\documentclass[10pt]{article} % For LaTeX2e
\usepackage{zelda_arxiv}
\usepackage{natbib}
\usepackage{enumitem}
\usepackage{amsmath,amsfonts,bm}

\def\eqref#1{equation~\ref{#1}}
\def\1{\bm{1}}

\DeclareMathAlphabet{\mathsfit}{\encodingdefault}{\sfdefault}{m}{sl}
\SetMathAlphabet{\mathsfit}{bold}{\encodingdefault}{\sfdefault}{bx}{n}

\newcommand{\E}{\mathbb{E}}

\newcommand{\softmax}{\mathrm{softmax}}

\DeclareMathOperator*{\argmin}{arg\,min}

\usepackage{hyperref}
\usepackage{url}

\usepackage{dblfloatfix}
\usepackage{subcaption}
\usepackage{hyperref}
\usepackage{xcolor}
\usepackage{amsthm,amsfonts,amsmath,amssymb,mathtools}
\usepackage{microtype}
\usepackage{hyperref}
\usepackage{xspace}
\usepackage{wrapfig}
\usepackage{marginnote}
\usepackage{enumitem}
\usepackage{booktabs}
\usepackage[noend]{algpseudocode}
\usepackage{algorithm}

\newtheorem*{rep@theorem}{\rep@title}
\newcommand{\newreptheorem}[2]{%                                                
\newenvironment{rep#1}[1]{%                                                     
 \def\rep@title{#2 \ref{##1}}%                                                  
 \begin{rep@theorem}}%                                                          
 {\end{rep@theorem}}}
\makeatother
\newreptheorem{theorem}{Theorem}
\newreptheorem{proposition}{Proposition}
\newreptheorem{lemma}{Lemma}
\newreptheorem{corollary}{Corollary}
\usepackage{etoolbox}
\usepackage{tikz}
\usetikzlibrary{tikzmark}
\usetikzlibrary{calc}

\errorcontextlines\maxdimen

\newcommand{\ALGtikzmarkcolor}{black}% customise this, if you want
\newcommand{\ALGtikzmarkextraindent}{4pt}% customise this, if you want
\newcommand{\ALGtikzmarkverticaloffsetstart}{-1ex}% customise this, if you want
\newcommand{\ALGtikzmarkverticaloffsetend}{-.5ex}% customise this, if you want
\makeatletter
\newcounter{ALG@tikzmark@tempcnta}

\newcommand\ALG@tikzmark@start{%
    \global\let\ALG@tikzmark@last\ALG@tikzmark@starttext%
    \expandafter\edef\csname ALG@tikzmark@\theALG@nested\endcsname{\theALG@tikzmark@tempcnta}%
    \tikzmark{ALG@tikzmark@start@\csname ALG@tikzmark@\theALG@nested\endcsname}%
    \addtocounter{ALG@tikzmark@tempcnta}{1}%
}

\def\ALG@tikzmark@starttext{start}
\newcommand\ALG@tikzmark@end{%
    \ifx\ALG@tikzmark@last\ALG@tikzmark@starttext
    \else
        \tikzmark{ALG@tikzmark@end@\csname ALG@tikzmark@\theALG@nested\endcsname}%
        \tikz[overlay,remember picture] \draw[\ALGtikzmarkcolor] let \p{S}=($(pic cs:ALG@tikzmark@start@\csname ALG@tikzmark@\theALG@nested\endcsname)+(\ALGtikzmarkextraindent,\ALGtikzmarkverticaloffsetstart)$), \p{E}=($(pic cs:ALG@tikzmark@end@\csname ALG@tikzmark@\theALG@nested\endcsname)+(\ALGtikzmarkextraindent,\ALGtikzmarkverticaloffsetend)$) in (\x{S},\y{S})--(\x{S},\y{E});%
    \fi
    \gdef\ALG@tikzmark@last{end}%
}

\apptocmd{\ALG@beginblock}{\ALG@tikzmark@start}{}{\errmessage{failed to patch}}
\pretocmd{\ALG@endblock}{\ALG@tikzmark@end}{}{\errmessage{failed to patch}}
\makeatother
\usepackage[capitalize,noabbrev]{cleveref}

\definecolor{darkblue}{rgb}{0.0,0.0,0.55}
\hypersetup{  
  colorlinks = true,
  citecolor  = darkblue,
  linkcolor  = darkblue,
  citecolor  = darkblue,
  filecolor  = darkblue,
  urlcolor   = darkblue,
}
\theoremstyle{plain}
\newtheorem{theorem}{Theorem}[section]
\newtheorem{proposition}[theorem]{Proposition}

\theoremstyle{definition}
\newtheorem{definition}[theorem]{Definition}

\theoremstyle{remark}
\newtheorem{remark}[theorem]{Remark}
\usepackage{thm-restate}

\DeclarePairedDelimiterX{\bdx}[2]{[}{]}{%
  #1\;\delimsize\|\;#2}
\newcommand{\bd}{D\bdx}

\newcommand{\V}{\mathbb V}
\newcommand{\deq}{:=}
\newcommand{\dm}{\mathcal E}
\newcommand{\iid}{\textit{i.i.d.}\xspace}

\newcommand{\eg}{\textit{e.g.}\xspace}

\newcommand{\nlsum}{\sum\nolimits}

\newcommand{\kl}{\textup{KL}}

\newcommand{\scal}[2]{\langle #1, #2\rangle}

\usepackage{tikz-qtree}

\title{Ensembling over Classifiers: a Bias-Variance Perspective}

\author{
        \name \!\!Neha Gupta\thanks{Work done at Google.} \email{nehagupta@cs.stanford.edu}\\
        \name Jamie Smith \email{jamieas@google.com}\\
        \name\!\!\! Ben Adlam \email{adlam@google.com}\\
        \name Zelda Mariet \email{zmariet@google.com}\\
}
\def\month{MM}  % Insert correct month for camera-ready version
\def\year{YYYY} % Insert correct year for camera-ready version
\def\openreview{\url{https://openreview.net/forum?id=XXXX}} % Insert correct link to OpenReview for camera-ready version

\begin{document}

\maketitle

\begin{abstract}
Ensembles are a straightforward, remarkably effective method for improving the accuracy, calibration, and robustness of models on classification tasks; yet, the reasons that underlie their success remain an active area of research. We build upon the extension to the bias-variance decomposition by~\citet{generalized_bvd} in order to gain crucial insights into the behavior of ensembles of classifiers. Introducing a dual reparameterization of the bias-variance tradeoff, we first derive generalized laws of total expectation and variance for nonsymmetric losses typical of classification tasks. Comparing conditional and bootstrap bias/variance estimates, we then show that conditional estimates necessarily incur an irreducible error. Next, we show that ensembling in dual space reduces the variance and leaves the bias unchanged, whereas standard ensembling can arbitrarily affect the bias. Empirically, standard ensembling \emph{reduces} the bias, leading us to hypothesize that ensembles of classifiers may perform well in part because of this unexpected reduction. We conclude by an empirical analysis of recent deep learning methods that ensemble over hyperparameters, revealing that these techniques indeed favor bias reduction. This suggests that, contrary to classical wisdom, targeting bias reduction may be a promising direction for classifier ensembles.
\end{abstract}

\section{Introduction}
The success of deep learning has catalyzed many extensions of classical learning theory to modern models and algorithms. In particular, the double-descent phenomenon \citep{Belkin15849} has inspired a wide breadth of research, in which the classical bias-variance (BV) decomposition has been key~\citep{neal2019modern,adlam2020}. Such analyses have been instrumental to understanding the performance of various strategies that improve deep learning methods, including ensembles of deep networks~\citep{nnens, deep-ensembles} --- a simple method with state-of-the-art robustness and uncertainty results~\citep{ovadia2019, gustafsson2019evaluating}.  % , bagging \cite{breiman1996bagging}, and boosting~\cite{schapire1990strength, freund1997decision}.

Most bias-variance analyses are specific to the mean-squared error (MSE) loss. Although it is possible to analyze classifiers from an MSE perspective~\citep{yang2020rethinking}, such a restriction inevitably reduces the power of our analysis. However, generalizing bias-variance decompositions to non-MSE losses is challenging. Although the MSE allows a decomposition based on the performance of the mean predictor, this is a peculiarity of the MSE rather than the rule. In the general case, bias-variance decompositions require manipulating a ``central prediction'' which is much less amenable to analysis. And although bias-variance decompositions can be derived for losses that take the form of a Bregman divergence~\citep{generalized_bvd}, such decompositions are difficult to interpret, and appear to break away from standard intuitions.

% Furthermore, even empirical behavior of ML models departs from classical understanding when viewed through the lens of a non-MSE bias-variance decomposition. For example, we will show that even simply ensembling over random seeds of a DL model reduces not only the variance, but also the bias of the predictor --- a flagrant departure from MSE-based behaviors.

We begin by bridging the gap between the classical bias-variance decomposition and its generalization to losses such as the KL divergence. Analyzing the properties of the bias and variance for non-symmetric losses, we characterize their departure from standard intuition, and show that alternate ensembling techniques recover the standard behavior of ensembles under the MSE. Crucially, our analysis suggests that ensembles of classifiers may be particularly effective due to their ability to reduce the bias as well as the variance.

Based on this theoretical analysis, we investigate recent promising techniques that ensemble over the hyperparameters of deep learning algorithms~\citep{wenzel2020hyperparameter,zaidi2020neural}. We will see that, by augmenting the hypothesis space, these methods achieve higher bias reduction but comparable or worse variance reduction than deep ensembles.

\paragraph{Contributions.} Using a dual reparameterization of the central prediction in the bias-variance decomposition of~\citet{generalized_bvd}, we tease apart which behaviors of the bias and variance are specific to the mean squared error,\footnote{Or, more accurately, to \emph{symmetric} losses.} and which behaviors are common to all Bregman divergences. Our key contributions are the following.
\begin{itemize}[leftmargin=*]
\itemsep0em
    \item The central prediction for arbitrary Bregman divergences can be viewed as the primal projection of the mean prediction in a dual space defined by the loss itself.
    \item The variance for any arbitrary Bregman divergence satisfies a generalized law of total variance.
    \item Conditional estimates of the bias and variance are biased by an irreducible quantity; iterative bootstrapping can improve these estimates.
    \item Ensembling predictions in dual space recovers the behavior of ensembles under the squared Euclidean loss.
    \item Vanilla ensembling can increase or decrease the bias; empirically, vanilla ensembling \emph{reduces} the bias.
    \item Recent ensemble methods show lower bias and higher variance than ensembling over random seeds.
\end{itemize}
\section{Related work}
\paragraph{Bias-variance decompositions.} The bias-variance tradeoff has been an important tool in understanding the behavior of machine learning models~\citep{geman1992neural, kong1995error, breiman1996bias, adlam2020,yang2020rethinking, dascoli2020double,neal2019modern}. Key to the decomposition are the notions of a ``central label'' (if label noise exists) --- and of a ``central prediction''.  Most analyses focus on the bias-variance decomposition for the Euclidean square loss; in this case, the central label and prediction correspond respectively to the \emph{mean} label and prediction. \citet{james2003variance} proposed a more general decomposition for symmetric losses,~\citet{Domingos00aunified} focussed on the 0-1 loss, and~\citet{generalized_bvd} proposed a generalization to the space of all Bregman divergences. \cite{hansen2000general} identify which loss functions admit bias-variance decompositions with specific properties (but do not analyze the decompositions themselves). \cite{jiang2017generalized, buschjager} decompose twice differentiable losses via second order Taylor expansion. In the general case --- including for the KL divergence --- the resulting decompositions are approximate. Additionally, the variance term may depend on target labels, in a significant departure from any standard definition of a variance.

\paragraph{Bregman divergences.} Bregman divergences are a generalization of the notion of distance, similar to but less restrictive than metrics. Bregman divergences and operations in their associated dual space are instrumental to optimization techniques such as mirror descent and dual averaging~\citep{nemirovski1983problem,DBLP:journals/mp/Nesterov09,juditsky2021unifying}. Closer to our work, the Bregman representative defined in~\citet{bregman-clustering} is closely related to the central label defined in~\citep{generalized_bvd}.

\paragraph{Ensembles of deep networks.} Ensembling combines the predictions of multiple models to improve upon the performance of a single model; see, \eg,~\citep{zhou2019ensemble, ensembles-in-ml}. Recently, ensembling over neural networks which only differ in their random seed (``deep ensembles'') has been shown to be a particularly strong baseline for a variety of benchmarks~\citep{deep-ensembles}. This result, in turn, prompted further research into alternative ensemble methods, including ensembling over hyper-parameters~\citep{wenzel2020hyperparameter}, architectures~\citep{zaidi2020neural}, and joint ensemble training~\citep{webb2020ensemble}. In parallel, several hypotheses have been proposed to explain the performance of deep ensembles. \citet{fort2019deep} showed empirically that deep ensembles explore diverse modes in the loss landscape. \citet{allen2020towards} argued that the effectiveness of deep ensembles hinges on the assumption that inputs to the model have multiple correct features, which will be learned by the different ensemble members. \citet{wilson2020bayesian, hoffmann2021deep} analyzed deep ensembling as a Bayesian averaging procedure. \citet{masegosa,ortega2021diversity} focus on understanding the relationship between the generalization of neural networks and the diversity for deep ensembles, both from theoretical and empirical perspectives. \citet{lobacheva2020power,kobayashi2021reversed} considered the interplay between model and ensemble size. 

\paragraph{Results specific to the KL divergence.}  Due to its importance in machine learning, the KL divergence (and thus, the cross-entropy loss) is one of the few non-MSE losses for which the bias-variance tradeoff has been specifically analyzed~\citep{heskes1998bias, yang2020rethinking}. For the KL divergence, the ``central predictor'' corresponds to an average in log-probability space, which has been studied in many works, including~\citep{Brofos2019ABD,webb2020ensemble}. The fact that the bias remains unchanged when averaging predictions in log-probability space has been mentioned briefly in~\citep{diettrich}.

% Building on these insights, we show that the averaging in the log probability space keeps the bias unchanged and only reduces the variance whereas the current practice of averaging probabilities in the probability space affects affects both the bias and the variance.

%\cite{lee2020finite}

%\input{sections/bias-variance}
%\input{sections/bootstrap}
%\input{sections/empirical}

\section{Bias-Variance decomposition}
 We begin with some background material regarding Bregman divergences, which also serves to set our notation. We refer the interested reader to \citet{cesa2006prediction} for more background on Bregman divergences and proofs. % Note that both the commonly used losses in machine learning, KL divergence and the mean squared error belong to the class of Bregman divergences.

Our theoretical results are presented for arbitrary Bregman divergences; proofs not provided in the main text can be found in \cref{app:proofs}. Our experimental results use the KL divergence (which corresponds to the cross-entropy loss for one-hot labels), which is standard for classification models.
\subsection{Preliminaries}
Let $\mathcal X$ be a closed, convex subset of $\mathbb R^d$, and let $F: \mathcal X \to \mathbb R$ be a strictly convex, differentiable function over $\mathcal X$. The Bregman divergence associated with $F$ is the function $D_F: \mathcal X \times \mathcal X \to \mathbb R^+$, such that
\begin{equation}
\label{eq:bd}
    D_F[y \| x] \deq F(y) - F(x) - \nabla F(x)^\top (y-x).
\end{equation}
It follows directly from the convexity of $F$ that $D_F$ is convex in its first argument, although not necessarily in its second argument~\citep{bregman-convex}.

The Bregman divergence of the convex conjugate $F^*$ of $F$ will also be of particular importance. We recall that the convex conjugate of a convex function $F$ is defined as
\[F^*(z) = \sup_x~\langle z, x \rangle - F(x).\]
As $F$ is differentiable, we denote by $x^* = \nabla F(x)$ the \emph{dual} of $x \in \mathcal X$; in particular, $x = (x^*)^* = \nabla F^*(\nabla F(x))$. 

A Bregman divergence $D_F$ and its conjugate equivalent $D_{F^*}$ are related by the following equality:
\begin{proposition}[{\citet{cesa2006prediction}}]
\label{prop:bregman-dual}
    $\forall x, y \in \mathcal X, D_F\bdx x y = D_{F^*}\bdx {y^*} {x^*}$.
\end{proposition}

\subsection{General statement}
A bias-variance analysis decomposes the average divergence between two \emph{independent} random variables, $\E D[Y\|X]$, into a bias and two separate variance terms. The bias measures the divergence between the average label and the average prediction; the variances measure the amount of fluctuation in the labels (Bayes error) and predictions (model variance). 

These variances measure fluctuations around an ``average'', or ``central'' variable. Under the mean squared error, these central variables are the expected label and expected prediction. For Bregman divergences more generally, these ``average'' labels and predictions are the minimizers of the expected Bregman divergence. 
\begin{definition}[Central label]
\label{def:central-label}
  Let $Y$ be a random variable over $\mathcal X$ (intuitively, the label). We call \emph{central label} the unique minimizer $\argmin_{z \in \mathcal X} \E D[Y \| z]$.
\end{definition}
\begin{definition}[Central prediction]
  Let $X$ be a random variable over $\mathcal X$ (intuitively, the prediction). We call \emph{central prediction} the unique minimizer $\argmin_{z \in \mathcal X} \E D[z \| X]$.
\end{definition}

\begin{proposition}[{\citet{bregman-clustering}\!\!}
 \label{prop:mean}]
 The central label satisfies $\argmin_{z \in \mathcal X} \E D[Y \| z] = \E Y$.
\end{proposition}

By analogy to \cref{prop:mean}, and for reasons that will appear clear momentarily, we will refer to the minimizer $z = \argmin_z \E D[z \| X]$ as the \emph{dual mean}, and write it $\dm X$.

%Although the central label is simply the mean $\E Y$~\citep{bregman-clustering}, the central prediction seems less approachable,. 
We can now write out the bias-variance decomposition for any Bregman divergence $D$~\citep{generalized_bvd}:
% \begin{equation}
%     \label{eq:bv}
%     \E D[Y\|X] = \underbrace{\E D[Y\|\E Y]}_{\textup{label noise/variance}} + \underbrace{D[\E Y\| \dm X]}_{\textup{bias}} + \underbrace{\E D[\dm X\|X]}_{\textup{model variance}}.
% \end{equation}
\begin{align}
    \E D[Y\|X] =&~ \E D[Y\|\E Y] && \textit{(Bayes error)} \notag \\
    &+ D[\E Y\| \dm X] && \textit{(bias)} \label{eq:bv} \\
    &+ \E D[\dm X\|X].&& \textit{(model variance)} \notag
\end{align}
%The first term captures the variation in the random variable $Y$ around its average value $\E Y$. The second term captures the bias, and the third term captures the variation in the random variable $X$ around its ``central'' value $\dm X$. 

%For example, when $D$ is the MSE, both variance terms of \cref{eq:bv} take the classical form $\V Y = \E(Y - \E Y)^2$, and $\V X = \E (X - \E X)^2$.
%When $D$ is a training or evaluation loss applied to a machine learning model, $Y$ will typically be the true label fo a training point, and $X$ the model's corresponding prediction (\eg, for the cross-entropy loss). 

% When $D$ is the MSE, both variance terms of \cref{eq:bv} take the classical form $\V Y = \E(Y - \E Y)^2$, and $\V X = \E (X - \E X)^2$. 
Because Bregman divergences are not guaranteed to be symmetric, \cref{eq:bv} takes a more complicated form than the mean squared error decomposition: central label and central prediction are computed differently, and the ordering of terms within the bias and variances is now important.
%It is crucial to note how asymmetry of the divergence $D$, and hence the ordering of label $Y$ and prediction $X$, affects the decomposition. The central terms are either the mean $\E Y$ or the dual mean $\dm X$, and are either the first or the second argument of the Bregman divergence.

The main obstacle in bias-variance decompositions for Bregman divergences lies in the form of the central prediction, $\argmin_z \E \bd z X$; the central label remains easily interpretable as $\argmin \E \bd Y z = \E Y$.

Our first contribution resolves this difficulty via a simple observation: the central prediction $\dm X$ has a straightforward intepretation when we leverage the dual space defined by the convex conjugate of $F$.
\begin{proposition}[Dual mean]
\label{prop:dual-mean}
  The dual mean $\dm X$ is the primal projection of the mean of $X$ in dual space: \[\dm X = (\E X^*)^*.\]
\end{proposition}
\begin{proof}
By simple application of \cref{prop:bregman-dual}.
\begin{align*}
    \argmin_{z \in \mathcal X} \E D_F\bdx z X &= \argmin_{z \in \mathcal X} \E D_{F^*} \bdx {X^*} {z^*} = \Big(\argmin_{z^* \in \mathcal X^*} D_{F^*}\bdx {X^*} {z^*}\Big)^* = (\E X^*)^*. \tag*{\qedhere}
\end{align*}
\end{proof} 
This reformulation is crucial to our analysis, and, to the extent of our knowledge, novel.\footnote{\citet{generalized_bvd} showed the characterization $z = \argmin_z \E D[z \| X]$ satisfies $\nabla F(z) = \E[\nabla F(X)]$, which also entails Prop.~\ref{prop:dual-mean}.} 
\begin{remark}
  For $F(x) = \|x\|_2^2$, $D_F$ is the MSE, and $\dm X = \E X$. More generally, $\E X = \dm X$ if $D_F$ is symmetric. 
\end{remark}
\begin{remark}
  When $F$ is the negative entropy over the probability simplex, $D_F$ is the KL divergence, and \[\dm X = \softmax (\E \log X).\]
\end{remark}
\begin{remark}
  \label{rmk:diversity}
  Reordering \cref{eq:bv} for the KL divergence, we can write the loss of an ensemble of models averaged in log-probability space as the average loss of each model minus the variance term, thus recovering the ensemble diversity regularizer from~\citep{webb2020ensemble}.
\end{remark}
% When $D$ is the KL divergence, the bias-variance decomposition reduces to 
% \begin{equation}
%     \label{eq:bv-kl}
%     \E \kl[Y\|X] = \E \kl[Y\|\E Y] + \E \kl[\E Y\| \dm X] + \E \kl [\dm X\|X].
% \end{equation}
% where the dual mean is defined as $\dm X = \softmax (\E \log X)$. 

\subsection{Laws of total expectations and variance}
Despite its more general form, the model variance $\V X = D[\dm X \| X]$ in \cref{eq:bv} satisfies fundamental properties associated with the standard variance $\E (X - \E X)^2$. In particular, one can easily verify that $\V X \ge 0$, and that $\V X = 0$ if and only if $X$ is almost surely constant.

We next show that $\V X$ also follows a generalization of the law of total variance. This law disentangles the effect of different sources of randomness, and is thus a fundamantal tool in model analysis. Given two variables $X$ and $Z$, the law of total variance decomposes the standard (Euclidean) variance as $\V X = \E[\V (X \mid Z)] + \V [\E(X \mid Z)]$: the variance of $X$ is the sum of the variances respectively \emph{unexplained} and \emph{explained} by $Z$.  

We begin by showing that the dual mean satisfies its own form of the law of total expectation; proving this is straightforward with the reparameterization of \cref{prop:dual-mean}.
\begin{restatable}{lemma}{iteratedexpectations}
  \label{prop:iterated-expectations}
  Let $X,Z$ be random variables on $\mathcal X$. Then $\dm X = \dm_Z [\dm [X | Z]]$, where $\dm [X \mid Z] \deq (\E_{X | Z} [X^*])^*$.
\end{restatable}
\begin{proof} The proof follows directly from the (standard) law of total expectation and \cref{prop:dual-mean}.
\[\dm_Z (\dm_{X|Z} X) = \dm_Z\Big[(\E_{X|Z}X^*)^*\Big] = (\E_Z \E_{X|Z} X^*)^* = (\E X^*)^* = \dm X. \tag*{\qedhere}\]
\end{proof}
With \cref{prop:iterated-expectations} in hand, we can easily show a generalized form of the law of total variance, which simply accounts for the definition of dual mean and generalized variance.
\begin{restatable}{lemma}{dualtotalvariance}
\label{lem:tv_2}
Let $X, Z$ be random variables over $\mathcal X$. The variance $\V X \deq \E D[\dm X \| X]$ satisfies a generalized law of total variance.
  \[\V[X] =  \E[\V [X | Z]] + \V[\dm [X|Z]].\]
\end{restatable}

We cannot overstate the importance of \cref{lem:tv_2}, which is key to disentangling sources of randomness in ML algorithms~\citep{neal2019modern,adlam2020}.
% In our setup, we will think of the random variable $Y$ as the label for any fixed data point $x$ and hence, consider the simplified setup where there is no randomness in $Y$. The second random variable $X$ will be prediction of the random model $f$ output by some random algorithm on the data point $x$. In this case, the bias variance decomposition reduces to the following form. Note that this result is not new, and has been mentioned for example in~\citep{yang2020rethinking}.

% \begin{theorem}
% \label{thm:ce-bv}
% Let $(x, y) \in \mathcal X \times \mathbb R^d$ be a (input, label) pair where $y$ is a deterministic classification label represented in its one-hot form. Let $\mathcal D$ be a distribution over predictions: $f(x) \sim \mathcal D$. Then, the bias-variance decomposition~\eqref{eq:bv} takes the form
% \[\E \ce[y \| f(x)] = \ce[y \| f^*(x)] + \E \kl[f^*(x) \| f(x)],\]
% where $f^*(x) = \dm[f(x)] = \textup{softmax}(\exp(\E \log f(x)))$.
% \end{theorem}

% Using the law of total variance, these variance can then be further decomposed  as $\V Y = \E \V [Y|Z] + \V \E[Y|Z]$: the variance of $Y$ is the sum of the variances respectively \emph{unexplained} and \emph{explained} by $Z$. 

\section{Conditional and bootstrapped estimates}
\begin{figure*}[t]
\centering
\includegraphics[width=.6\textwidth]{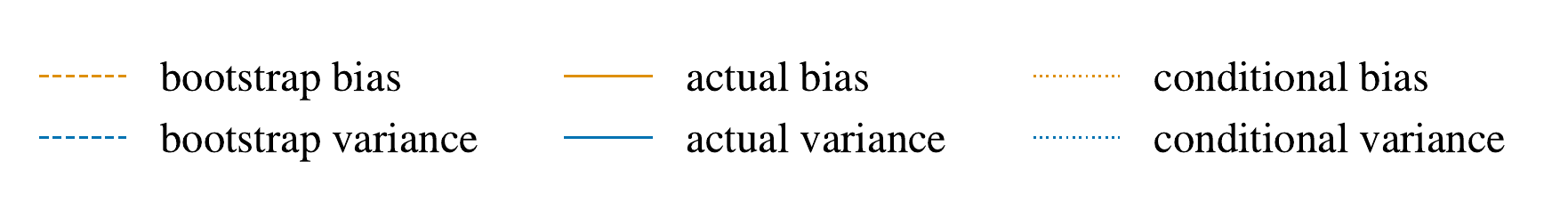}\\
\begin{subfigure}{.32\textwidth}
\vspace{-1em}
\includegraphics[height=1.6in]{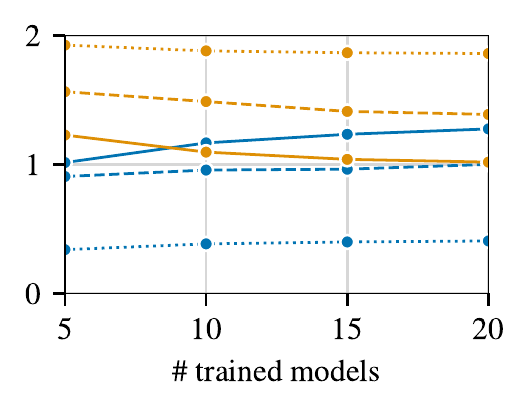}
\caption{CIFAR-10, 50 partitions}
\label{fig:results-bv-bootstrap-partition-cifar}
\end{subfigure}
\begin{subfigure}{.33\textwidth}
\vspace{-1em}
\includegraphics[height=1.6in]{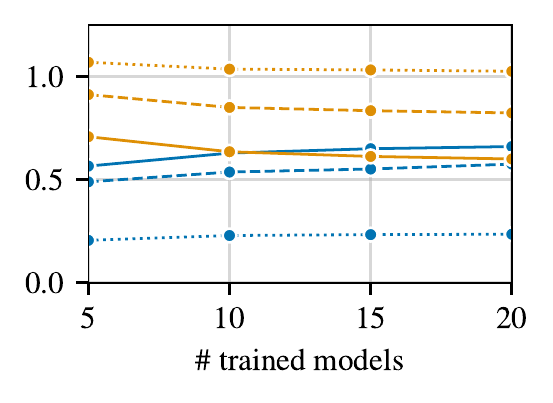}
\caption{CIFAR-10, 20 partitions}
\label{fig:results-bv-bootstrap-partition-cifar-20}
\end{subfigure}
\begin{subfigure}{.33\textwidth}
\vspace{-1em}
\includegraphics[height=1.6in]{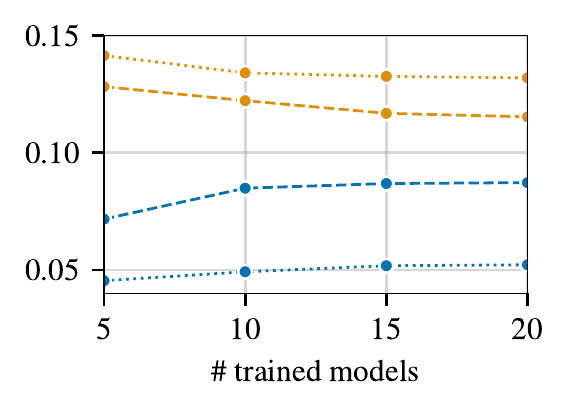}
\caption{CIFAR-10, no partitions}
\label{fig:results-bv-bootstrap-full-cifar}
\end{subfigure}
\caption{\small Conditional and bootstrapped estimates of the bias and variance on variations of the CIFAR-10 dataset. The approximations of the CIFAR-10 dataset in figures (a) and (b) allow us to compute the true bias and variance; in both cases, the bootstrap estimates are more accurate than the conditional estimates. Figure (c) shows the bootstrap and conditional estimates on the true CIFAR-10 dataset, for which the true estimate cannot be computed.}
\label{fig:results-bv-bootstrap}
\end{figure*}
The bias-variance tradeoff of \cref{eq:bv} applies to any source of randomness. This includes the dependency on the random seed chosen to train the model, as well as the randomness in the training data $T$. However, although we can easily draw a new random seed and retrain any model, sampling a true new training set is often more difficult, if not completely impossible. As deep models are notoriously data hungry, it is often optimal to train on all the training points available: we only get one draw of the training set.

When there is randomness that cannot be controlled for, the empirical estimates of the bias and variance in \cref{eq:bv} will necessarily be only approximate. If we only get one draw of the training set, for example, our estimates will be \emph{conditioned} on the available training data. The following proposition quantifies to which extent these conditional estimates depart from their unconditional equivalents.

\begin{restatable}{proposition}{propconditionalsecond}
\label{prop:conditional}
  Let $X, Z$ be two random variables. Applying \cref{eq:bv} to $X | Z$ then taking the expectation over $Z$ yields an alternate decomposition of the expected Bregman divergence:
    \[\E D[y \| X] = \underbrace{\E_Z D[y \| \dm(X|Z)]}_{\textup{Conditional bias: Bias}_Z} + \underbrace{\E_{Z}\E_{X|Z}\Big[D[\dm (X| Z) \| X]\Big|Z\Big]}_{\textup{Conditional variance: Var}_Z}.\]
  The conditional bias (resp. variance) overestimates (resp. underestimates) their respective total values by the fixed quantity $\E_Z D[\dm X \| \dm (X|Z)]$:
    \[\textup{Bias}_Z = \textup{total bias} + \E_Z D[\dm X \| \dm (X|Z)] \qquad \qquad  \textup{Var}_Z = \textup{total variance} - \E_Z D[\dm X \| \dm (X|Z)].\]
%\end{proposition}
\end{restatable}

\begin{remark}
The quantity $ \E_Z D[\dm X \| \dm (X|Z)]$ is non-negative, and equal to 0 if $X$ and $Z$ are independent.
\end{remark}

A alternative to the conditional estimates instead partitions the entire training set into disjoint subsets. Models are trained on these smaller subsets, which act as different draws of the training distribution; this approach yields an unbiased estimator that is consistent as the number of partitions goes to infinity.\footnote{See \cref{fig:partitions} in App.~\ref{app:partition} for an empirical analysis of how many partitions are necessary for  estimates to converge.} However,, partitioning raises the difficult question of how (and if) the model itself should be also modified: should the width or number of layers be reduced to accommodate the smaller training set? Furthermore, this approach is by construction incapable of estimating the bias and variance of the models on the full training set.

A second option lies in classical bootstrapping~\citep{efron1994introduction, edgeworth}: we create new datasets by sampling with replacement from the original dataset, and use these samples to estimate the bias and variance. Repeating this procedure by \emph{re}sampling from the bootstrap samples in turn estimates the quantity required to correct~\citep{edgeworth} the bootstrapped estimate (\cref{alg:bootstrap,fig:bootstrap}). This approach does not require reducing the size of the training set, but bootstrapped samples will contain duplicate points.

To compare the conditional and bootstrapped estimates, we trained wide WRNs with the cross-entropy loss on different partitionings of the CIFAR-10 dataset. \cref{fig:results-bv-bootstrap-partition-cifar} uses 50 partitions of 1k training points, and \cref{fig:results-bv-bootstrap-partition-cifar-20} uses 20 partitions of 2.5k training points. These partitions allow us to also compute the true bias and variance of the algorithm.\footnote{We insist here that these are the bias and variance of the algorithm trained on respectively 1k and 2.5k points, rather than the bias and variance of the algorithm trained on the entire dataset.} 

In both cases, we see that the bootstrapped estimates are significantly more accurate than the conditional samples, but both estimation methods systematically underestimate the variance and overestimate the bias. Additionally, the boostrap bias is more accurate than the bootstrap variance; for this reason, practitioners may prefer to estimate the variance as the total error minus the bootstrapped bias.

Finally,
\cref{fig:results-bv-bootstrap-full-cifar} shows the bootstrap and conditional estimates on the full dataset, for which we cannot compute the actual bias and variance. Since as we decrease the number of partitions, the gap between bias and variance widens, it is likely that the true bias dominates the variance in the non-partitioned regime (\cref{fig:results-bv-bootstrap-full-cifar} and \cref{app:partition}).

\begin{figure}
\begin{minipage}{.47\textwidth}
\vspace{1em}
\begin{tikzpicture}[nodes={draw, circle}, sibling distance=.8cm,->]
%\begin{tikzpicture}[nodes={draw, circle}, sibling distance=.9cm, xscale=0.9,yscale=0.75,>=]
\node{$T$}    
    child { node {$\scriptstyle T_1$} 
        child { node {$\scriptscriptstyle T_{11}$} }
        child { node {$\scriptscriptstyle T_{12}$} }
        child { node {$\scriptscriptstyle T_{11}$} }
    }
    child [missing]
    child [missing]
    child {node {$\scriptstyle T_2$} 
        child { node {$\scriptscriptstyle T_{21}$} }
        child { node {$\scriptscriptstyle T_{22}$} }
        child { node {$\scriptscriptstyle T_{11}$} }
    }
    child [missing]
    child [missing]
    child {node {$\scriptstyle T_3$} 
        child { node {$\scriptscriptstyle T_{31}$} }
        child { node {$\scriptscriptstyle T_{32}$} }
        child { node {$\scriptscriptstyle T_{33}$} }
    };
\end{tikzpicture}
\caption{\small Generating bootstrap samples from a dataset $D$. The child of every node is sampled with replacement from its parent so that all datasets have the same size. \vspace{-1em}}
\label{fig:bootstrap}
\end{minipage}~~~
\begin{minipage}{.5\textwidth}
\begin{algorithm}[H]
    \caption{Bootstrap estimate of bias (or variance)}
    \begin{algorithmic}
    \small
        \State {\bfseries Input:} Training set $T$, number of bootstrap samples $B$
        \For{$i \in \{1, \ldots, B\}$}
            \State $T_i \leftarrow \texttt{uniform\_sample}(T)$  \Comment{{\color{gray}Of size $|T|$}}
            \For{$j \in \{1, \ldots, B\}$}
                \State $T_{ij} \leftarrow \texttt{uniform\_sample}(T_i)$   \Comment{{\color{gray}Of size $|T|$}}
            \EndFor
            \State $b^{(2)}_i \leftarrow \textup{bias}(\{T_{ij}\}_j)$ \Comment{{\color{gray}Bootstrap estimate for $T_i$}}
        \EndFor
        \State $b^{(1)} \leftarrow \textup{bias}(\{T_i\}_i)$ \Comment{{\color{gray}Bootstrap estimate for $T$}}
        \State $b^{(2)} \leftarrow \frac 1B \sum_i b^{(2)}_i$
        \State $t \leftarrow b^{(1)} / b^{(2)}$ \Comment{{\color{gray}Corrective term}}
        \State $b^{(0)} \leftarrow t b^{(1)}$
    \State {\bfseries return} Bias estimate $b^{(0)}$.
    \end{algorithmic}
    \label{alg:bootstrap}
\end{algorithm}
\end{minipage}%
\end{figure}

Often, however, we care directly about the conditional estimates. This can be because bootstrapped estimates are computationally intensive, requiring training $\mathcal O(B^2)$ models where each model itself may be an ensemble. 

Alternatively, we simply might care about a training algorithm's performance on a specific training set; This is the case for standard benchmarking suites and many real-world applications where the training set cannot be modified, even when the test distribution itself shifts over time.

\section{Ensembles and the BV decomposition}
We conclude our theoretical analysis by considering ensembles within the context of the bias-variance decomposition for non-symmetric losses. 

\subsection{Primal ensembling}
Most often in deep learning, ensembling simply averages the outputs of $n$ models that differ in their initialization~\citep{deep-ensembles}. This ensembling is commonly motivated by the desire to reduce the variance of the predictive model. Indeed, for the MSE, we know that ensembling by averaging the outputs of models drawn in \iid fashion (a) reduces the variance, and (b) conserves the bias. 

We begin by recovering (a) under some additional weak convexity assumptions on the Bregman divergence, but (b) will prove impossible in the general case.

\begin{restatable}{proposition}{propvanillavariance}
  \label{prop:vanilla-variance}
  Let $D$ be a Bregman divergence that is \emph{jointly} convex in both variables. Let $X_1, \ldots, X_n$ be $n$ \iid random variables drawn from some unknown distribution, and define $\hat X = \frac 1n \sum_i X_i$. Then, 
  \[\E D[\dm \hat X \| \hat X] \le \E D[\dm X \| X].\]
\end{restatable}
Both KL divergence and mean squared error are jointly convex, and are special cases of \cref{prop:vanilla-variance}. 

However, conserving the bias under ensembling requires that the Bregman divergence be symmetric --- ensuring that $\E$ and $\dm$ are equivalent. In general, ensembling can either decrease \emph{or} increase the bias.
\begin{restatable}{proposition}{propbias}
  \label{prop:bias}
  Let $D$ be the KL divergence. There exists a distribution $\mathcal D$ over predictions $X \in \mathbb R^2$ and a label $y \in \{0, 1\}$ such that the divergence bias $D[y \| \dm[\cdot]]$ satisfies
  \begin{align*}
      D[y \| \dm \hat X] &< D[y \| \dm X] \\
      D[1-y \| \dm \hat X] &> D[1-y \| \dm X],
  \end{align*}
 where as above we define the random variable for ensemble predictions $\hat X = \frac 1n \sum_i X_i$, and by abuse of notation we conflate $y \in \{0, 1\}$ with its one-hot representation.
\end{restatable}

\begin{remark}
  Despite \cref{prop:bias}, ensembling reduces the total cross-entropy loss due to Jensen's inequality.
\end{remark}
That vanilla ensembling does not preserve the bias (and can, in fact, increase it!) is a strong departure from what one might naively expect. Given this, it is natural to seek an ensembling method that would maintain the following two behaviors: the variance decreases, and the bias is conserved.

\subsection{Dual ensembling}
Once again, our path forward is guided by the dual expectation $\dm$. To keep the bias unchanged by ensembling, it suffices to ensemble in such a way that the ensemble predictor $\hat X$ satisfies the equality $\dm \hat X = \dm X$.

\begin{restatable}{proposition}{propdualbias}
\label{prop:dual-bias}
  Let $D$ be any Bregman divergence. Let $X_1, \ldots, X_n$ be $n$ \iid random variables drawn from some unknown distribution, and define 
  \[\hat X = \Big(\frac 1n \sum\nolimits_i X_i^*\Big)^*.\] 
  This operation reduces the variance and conserves the bias: for any label $y \in \mathcal X$, we have
  \begin{align*}
      \bd y {\dm \hat X} &= \bd y {\dm X} \\
      \E \bd {\dm \hat X} {\hat X} & \le \E \bd {\dm X} X.
  \end{align*}
\end{restatable}
In contrast to the reduction in variance that occurs for vanilla ensembles, we no longer require that $D$ be jointly convex; the natural convexity of $D$ in its first argument is sufficient. We refer to the operation in \cref{prop:dual-bias} as \emph{dual} ensembling, and to that of \cref{prop:vanilla-variance} as \emph{primal} ensembling.

\begin{remark}
  Under the KL divergence, dual ensembling amounts to averaging in log-probability space (sometimes referred to as geometric averaging):
  \[\hat X = \softmax\Big(\frac 1n \sum\nolimits_i \log(X_i)\Big).\]
\end{remark}

\begin{figure*}[t]
  \centering
  \includegraphics{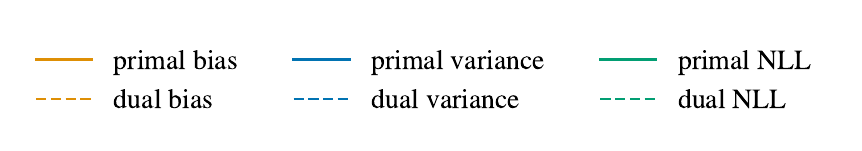}
  \begin{subfigure}{.49\textwidth}
    \vspace{-.5em}
    \centering
    \includegraphics[width=.8\textwidth]{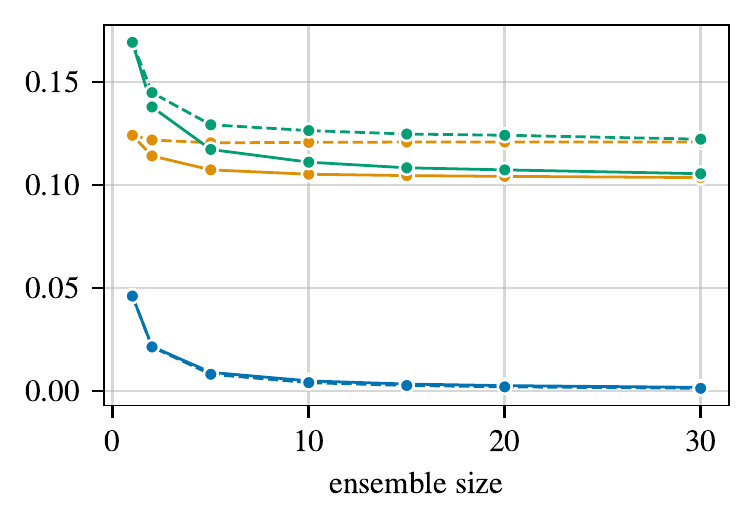}
    \caption{CIFAR-10}
  \end{subfigure}
  \begin{subfigure}{.49\textwidth}
  \centering
    \vspace{-.5em}
    \includegraphics[width=.8\textwidth]{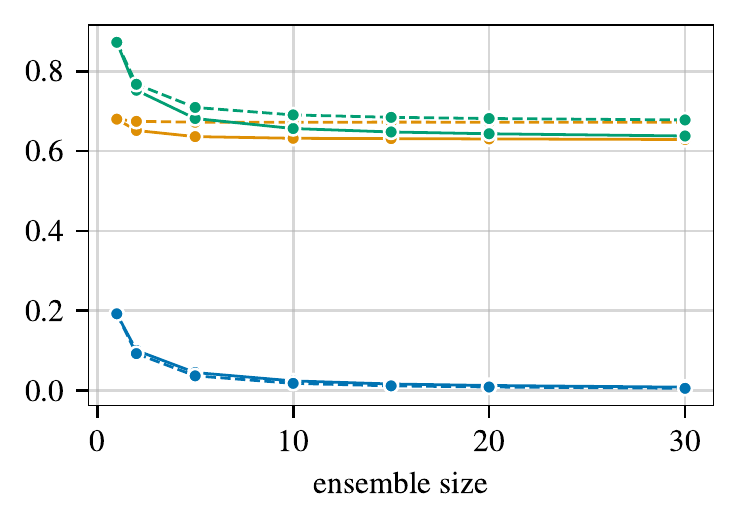}
    \caption{CIFAR-100}
  \end{subfigure}
  \caption{Conditional bias, variance, and NLL of primal and dual WRN-28--10 ensembles. As expected, the dual bias remains essentially constant as a function of the number of ensemble members, while the primal bias changes over time. As both primal and dual variance are reduced similarly by ensembling, the reduction in primal bias lets primal ensembling achieves better overall negative log likelihood on both datasets. The bias, variance and NLL are estimated by drawing 20 different draws of each ensemble size.}
 \label{fig:bv-cifar10}
\end{figure*}

Figure~\ref{fig:bv-cifar10} shows the evolution of the total loss, bias and variance when training ensembles of wide residual networks (WRN)~\citep{wrn} with the cross-entropy loss on CIFAR-10 and CIFAR-100.

As expected, the bias is overall unaffected by the size of the ensemble when dual ensembling, whereas the bias decreases under primal ensembling. For both ensembling methods, the variance reductions are comparable, and so primal ensembling achieves a lower NLL than dual ensembling. Furthermore, under primal ensembling, the decrease in variance is significantly higher than the decrease in bias. 

These results suggest the following hypotheses. Firstly, the empirical bias \emph{dominates} the variance in both experiments. As these are conditional estimates subject to the estimation error described in~\cref{prop:conditional}, we cannot affirm with certainty that the true bias domimates the true variance, although this is a plausible scenario based on the conclusions of \cref{fig:results-bv-bootstrap}. 

Secondly, primal ensembling may be successful in classification in part \emph{because} it affects the bias. Although theory does not guarantee that primal ensembling will reduce the bias, this appears to be the case in practice, thus providing a significant advantage to primal ensembling over over dual ensembling.

The following section analyzes under which conditions primal ensembling is favorable to dual ensembling.

\section{Empirical analysis of modern ensembles}

\begin{figure*}[t]
\centering
\includegraphics[width=.7\textwidth]{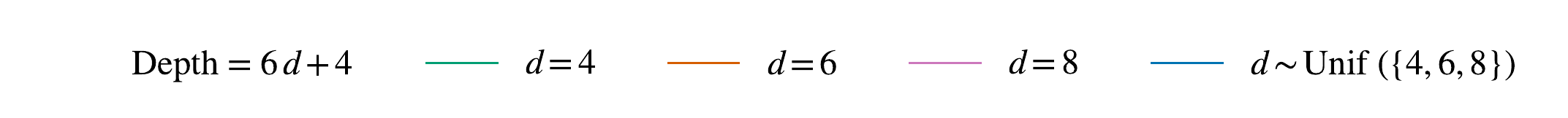}

    \begin{subfigure}{.32\textwidth}
        \centering
        \vspace{-1em}
        \includegraphics[width=\textwidth]{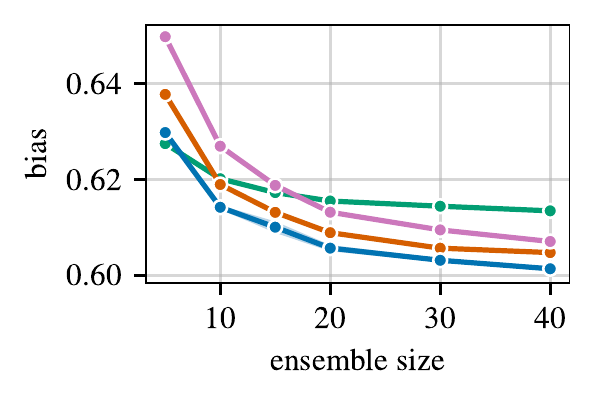}
        \caption{Bias}
        \label{fig:depth-bias}
    \end{subfigure}
    \begin{subfigure}{.32\textwidth}
        \vspace{-1em}       
        \centering
        \includegraphics[width=\textwidth]{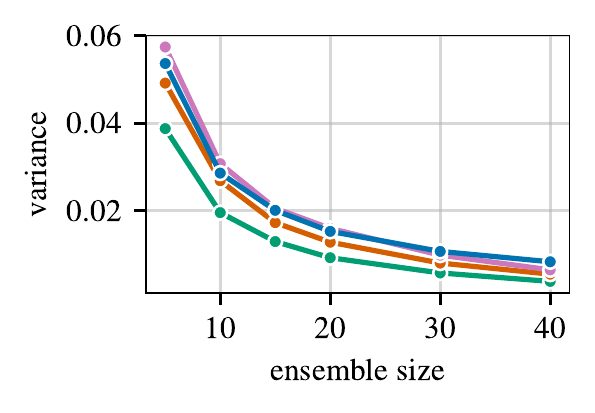}
        \caption{Variance}
        \label{fig:depth-variance}
    \end{subfigure}
    \begin{subfigure}{.32\textwidth}
        \centering
        \vspace{-1em}
        \includegraphics[width=\textwidth]{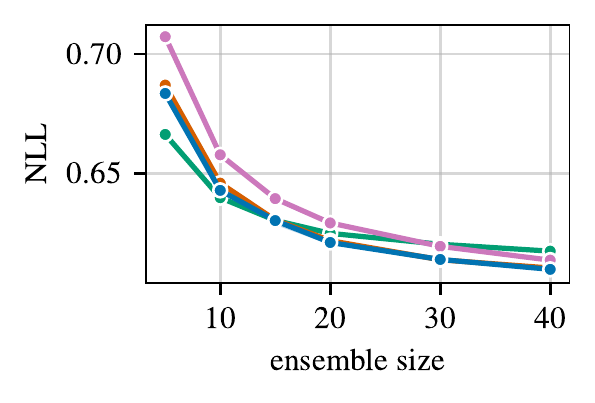}
        \caption{NLL}
        \label{fig:depth-nll}
    \end{subfigure}
    \caption{Bias, variance, error and NLL on CIFAR-100. Networks are trained with different initial random seeds and different depths and then ensembled in probability space (\emph{primal ensembling}), either for a fixed depth ($d=4, 6, 8$), or with depths sampled uniformly ($d \sim \textup{Unif}(\{4, 6, 8\})$). We average over three estimates of the bias and variance, each estimate using 5 draws of an ensemble. Ensembling over depths dramatically reduces the bias, but in turn increases the variance. However, the increase in variance is much smaller than the decrease in bias, and ensembles over multiple depths outperform ensembles of fixed depth -- even when the expected number of parameters is higher for fixed-depth ensembles.}
    \label{fig:depths}

    \centering
    \includegraphics{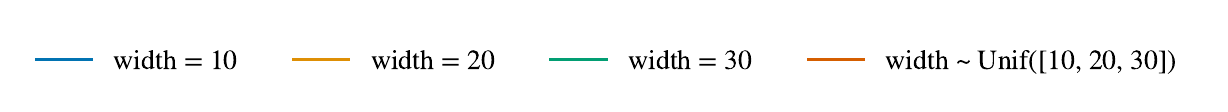}
    
    \begin{subfigure}{.32\textwidth}
        \vspace{-1em}
        \centering
        \includegraphics[width=\textwidth]{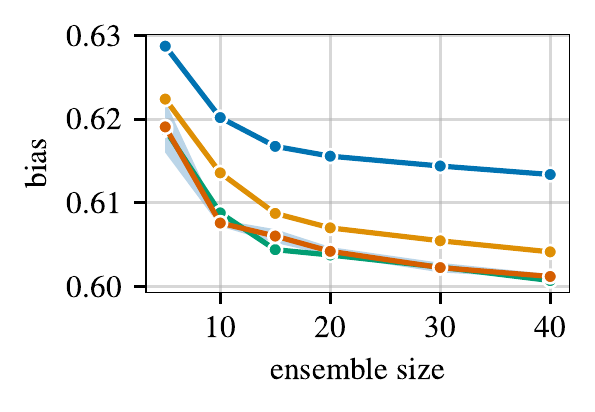}
        \caption{Bias}
        \label{fig:width-bias}
    \end{subfigure}
    \begin{subfigure}{.32\textwidth}
        \vspace{-1em}
        \centering
        \includegraphics[width=\textwidth]{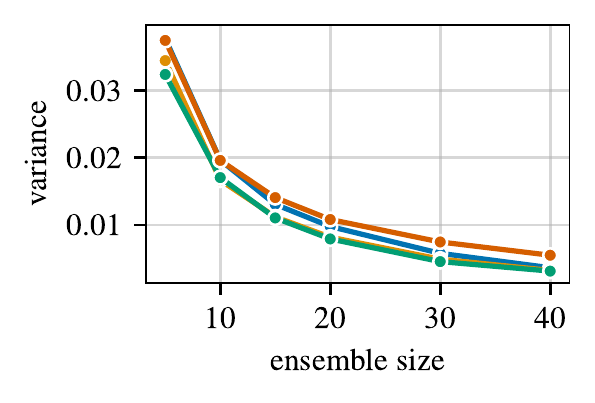}
        \caption{Variance}
        \label{fig:width-variance}
    \end{subfigure}
    \begin{subfigure}{.32\textwidth}
        \vspace{-1em}
        \centering
        \includegraphics[width=\textwidth]{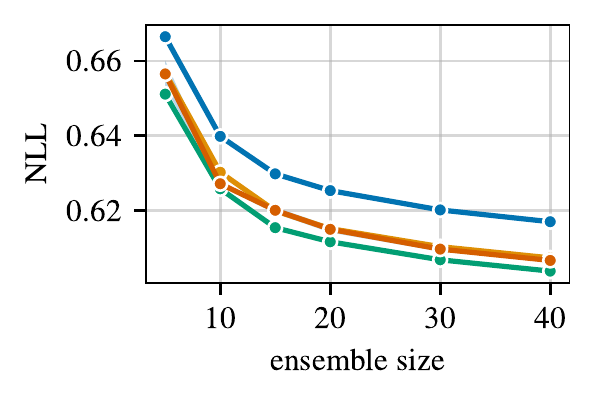}
        \caption{Cross entropy loss}
        \label{fig:width-nll}
    \end{subfigure}
\caption{Bias, variance and NLL on CIFAR-100. Networks are trained with different initial random seeds and different widths and then ensembled together in probability space (\emph{primal ensembling}), either for a fixed width (width $=10, 20, 30$) or with widths sampled uniformly (width $ \sim \textup{Unif}(\{10, 20, 30\})$.  We average over three estimates of the bias and variance, each estimate using 5 draws of an ensemble. As when varying the depth, ensembling over the width improves the bias at the cost of the variance. However, the bias improvements for ensembles of models with random widths are too low to improve over the overall performance of standard deep ensembles: the random width ensemble and the ensemble of width=$20$ have roughly the same number of parameters and similar NLLs.}
\label{fig:widths}
\end{figure*}

Different forms of networks ensembles have been shown to achieve state-of-the-art results~\citep{ensembles-in-ml,deep-ensembles,moe,ovadia2019,gustafsson2019evaluating}.  These ensembling techniques range from deep ensembles~\citep{deep-ensembles}, which ensemble over the uniform distribution over random seeds, to ensembling over larger hypothesis spaces, architectures, and hyperparameters~\citep{zaidi2020neural, simonyan2014very, he2016deep, antoran2020depth,wenzel2020hyperparameter}.

Our previous analysis covers deep ensembles (which sample ensemble members in \iid fashion over random seeds, as required by Props.~\ref{prop:vanilla-variance}, \ref{prop:bias} and \ref{prop:dual-bias}). To test whether improving the bias over the variance is a viable path to improving classification performance, we now focus on three ensembling techniques over hyperparameters: the width, depth, or regularization hyperparameters of the model. These techniques have been shown to be successful for several classification tasks~\citep{zaidi2020neural,wenzel2020hyperparameter}. % Secondly, by ensembling over larger hypothesis spaces, we expect a different effect on the bias and variance.

As we are interested in the effects of ensembling over bias and variance rather than improving state-of-the-art results, we reuse standard hyperparameters for a single WRN 28-10. When ensembling over learning hyperparameters~\citep{wenzel2020hyperparameter}, we reuse the hyperparameters of the original paper. Bias and variance are estimated using 5 draws of an ensemble, over 3 total runs for depth and width ensembles, and over 2 runs for hyperdeep ensembles due to computational costs. Additional details are provided in \cref{app:experimental}.

\paragraph{Ensembling over depths.} \citet{zaidi2020neural} showed that ensembling over different model architectures can yield a larger gain in accuracy compared to deep ensembles. We begin by ensembling over different depths, training 100 WRNs of depths of 28, 40, and 52 following the prescribed pattern depth$=6d+4$, and set the width multiplier to 10.

As expected, Figure~\ref{fig:depths} shows that ensembling over different depths, $d \sim \textup{Unif}(\{4, 6, 8\})$, decreases the bias significantly more than just increasing the depth for standard ensembles (\cref{fig:depth-bias}).

Ensembling over depths also unsurprisingly increases the variance of the models (\cref{fig:depth-variance}). Overall, though, the bias decrease is sufficient to improve the overall performance (\cref{fig:depth-nll}), and ensembles of random depths outperform ensembles of fixed depth $=40$ ($d=6$), which have roughly the same number of parameters.

\paragraph{Ensembling over widths.} We repeat the above process, fixing the depth to $6d+4=28$ and letting the width multiplier take values in $\{10, 20, 30\}$. We see in Figure~\ref{fig:widths} that ensembling across multiple widths is less successful than ensembling over depths: the bias reduction (\cref{fig:width-bias}) is not sufficient to compensate for the gain in variance (\cref{fig:width-variance}). 

\begin{wrapfigure}{r}{0.4\textwidth}
\centering
\vspace{-1em}
    \includegraphics[width=.4\columnwidth]{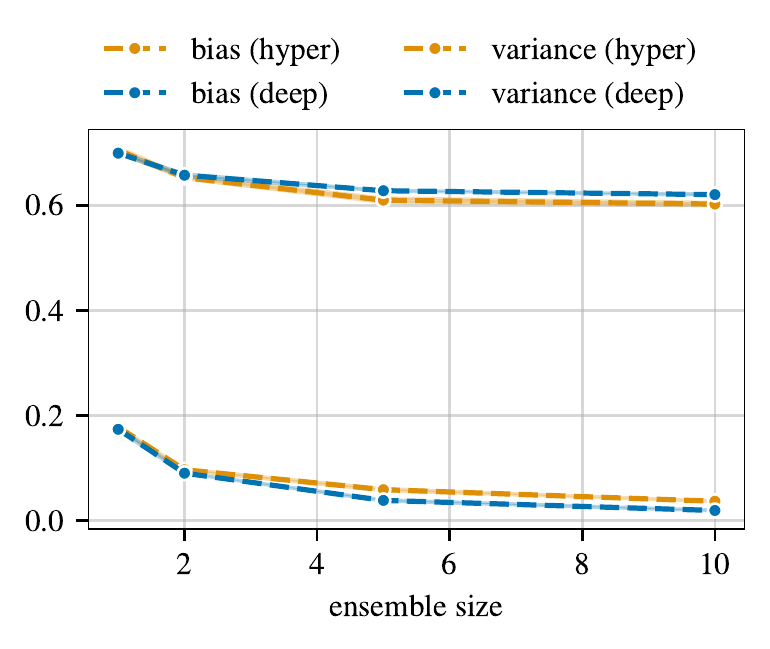}
    \caption{\small Deep vs. hyperdeep ensembles on CIFAR-100.}
     \label{fig:hyperdeep}
\vspace{-3em}
\end{wrapfigure}
Thus, ensembles of random widths perform on par with ensembles of fixed width $=20$ that have roughly the same number of parameters (\cref{fig:width-nll}).

\textbf{Hyperdeep ensembles.}~~ 
\cite{wenzel2020hyperparameter} showed that ensembling over networks trained with multiple hyperparameters lead to additional gain in accuracy.  As previously, we see in \cref{fig:hyperdeep} that ensembling over hyperparameters reduces the bias, at the cost of increasing the variance.

Overall, we see across all three experiments that ensembling over hyperparameters of the model yields lower biases and higher variances than ensembling over random seeds. These results are in line with previous literature~\citep{yang2020rethinking} as well as our intuition: increasing the hypothesis space of a model should reduce the bias but increase the variance.

\section{Conclusion and Future Work}
Ensembles of deep classifiers achieve state-of-the-art performance across a variety of benchmark tasks. Where ensembling has previously been analyzed for regression models through the lens of the bias-variance decomposition, applying this decomposition proves more complicated for  non-symmetric losses.

We begin by reformulating the bias-variance decomposition for Bregman divergences by~\citet{generalized_bvd}, showing that for any loss function that is a Bregman divergence, the key quantity required to define the bias and variance is the dual mean $\dm X = (\E X^*)^*$ of the prediction random variable $X$. When the loss is the mean-squared error, we recover $\dm X = \E X$, but this equality does not hold as soon as the loss is non-symmetric.
 
Using this reparameterization in dual space of the central prediction, we show that the model variance satisfies a generalized law of total variance, allowing us to disentangle the contributions of different sources of randomness on model performance. Unfortunately, it is impossible to estimate the bias and variance directly, as certain sources of randomness cannot typically be controlled. We show that the resulting \emph{conditional} estimates are guaranteed to overestimate the bias and underestimate the variance by a fixed quantity; however, these estimates can be improved by iterated bootstrapping.

The dual perspective on the bias-variance tradeoff also allows to introduce a theoretical framework for ensembling under which standard (regression) ensembling behavior is recovered in the classification setting. We show theoretically and empirically that ensembling in dual space will always reduce the variance and leave the bias unchanged. In primal space, however, ensembling will reduce the variance under gentle assumptions, but can have arbitrary effects on the bias. This leads us to hypothesise that, contrary to models trained with the Euclidean squared loss, there may be further room to improve ensembles of classification models by focusing on bias reduction over variance reduction.

We test these hypotheses by evaluating recent ensembling techniques that ensemble over different hyperparameters rather than simply ensembling over the training algorithm's random seed. As expected, we see that these methods reduce the bias significantly more than vanilla ensembles, albeit at the cost of a comparatively lesser increase in variance.

Finally, we note that the bias-variance perspective may be helpful to further understand how diversity affects the behavior of an ensemble, as suggested by \cref{rmk:diversity}.
\newpage
\clearpage
%\bibliography{refs}
%\bibliographystyle{zelda_bst}
\bibliography{refs}
\bibliographystyle{tmlr}
\clearpage
\appendix

\section{Proofs}
\label{app:proofs}
\begin{proposition}[Generalized triangle inequality for Bregman divergences]
  For any $x, y, z$ in the domain of $F$, we have $D(x, z) = D(x, y) + D(y, z) + \scal{\nabla F(y) - F(z)}{x-y}$.
\end{proposition}
\subsection{Bias-variance decomposition}
\propconditionalsecond*
\begin{proof}
Applying \eqref{eq:bv} to the conditional bias $\E_Z D[y \| \dm (X|Z)]$, we have
  \begin{align*}
      \E_Z \bd{y}{\dm (X\mid Z)} &= \bd[\Big] {y}{\dm [\dm (X|Z)} + \E_Z \bd[\Big]{\dm[\dm (X|Z)]}{\dm (X|Z)} \\
      &= \bd {y}{\dm X} + \E_Z \bd{\dm X}{\dm (X|Z)},
  \end{align*}
where the last equality stems from the law of iterated expectations for $\dm$, showing the equality for the bias terms. The result for the variance terms follows immediately, as conditional bias and variance have the same sum as the full bias and variance.
\end{proof}

%\begin{repproposition}{prop:conditional}
%   Let $X, Z$ be two random variables. We can use \eqref{eq:bv} to obtain the bias-variance decomposition of $X | Z$ and take the expectation over $Z$ to obtain an alternate decomposition of the expected Bregman divergence:
%  \[\E D[y \| X] = \underbrace{\E_Z D[y \| \dm(X|Z)]}_{\textup{Conditional bias}} + \underbrace{\E D[\dm (X| Z) \| X]}_{\textup{Conditional variance}}.\] 
%  The conditional bias (resp. variance) overestimates (resp. underestimates) its unconditional equivalent by the quantity $\E D[\dm X \| \dm (X|Z)]$:
%  \begin{align*}dm X] + \E D[\dm X \| \dm (X|Z)] \\
%      \textup{Variance}_{|Z} &= \E D(\dm X \| X) - \E D[\dm X \| \dm (X|Z)].
%  \end{align*}
%\end{repproposition}
%\rodolphe{Typo for the bias term}
%\begin{proof}
%Applying \eqref{eq:bv} to the conditional bias $\E_Z D[y \| \dm (X|Z)]$, we have
  %\begin{align*}
      %\E_Z D[y \| \dm (X|Z)] =& D[y \| \dm %[\dm (X|Z)]] + \E_Z D[\dm[\dm (X|Z)] \| %\dm (X|Z)] \\
      %&= D[y \| \dm X] + \E_Z D[\dm X \| \dm %(X|Z)],
%  \end{align*}
%where the last equality stems from the law of %iterated expectations for $\dm$, showing the equality for the bias terms. The result for %the variance terms follows immediately, as conditional bias and variance have the same %sum as the full bias and variance.
%\end{proof}

\subsection{Ensembling}
\propvanillavariance*
%\begin{repproposition}{prop:vanilla-variance}
%  Let $D$ be a Bregman divergence that is \emph{jointly} convex in both variables. Let $X_1, \ldots, X_n$ be $n$ random variables drawn in \iid fashion, and define $\hat X = \frac 1n \sum_i X_i$. Then, the variance $\E D(\dm[\cdot]\| \cdot)$ is reduced by ensembling:
%  \[\E D(\dm \hat X \| \hat X) \le \E D(\dm X \| X).\]
%\end{repproposition}
\begin{proof}
Let $D: S \times S \to \mathbb R^+$ be a Bregman divergence jointly convex in both variables. Let $\hat X = \frac 1n \sum_i X_i$, where the $X_i$ are \iid. By convexity, for any $z \in \mathcal X$,
\begin{align*}
    \bd{z}{\hat X} &\le \frac 1n \sum_i \bd{z}{X_i} \\
    \E \bd{z}{\hat X} &\le \frac 1n \sum_i \E \bd{z}{X_i} = \E \bd{z}{X} \\
    \min_z \E \bd{z}{\hat X} &\le \min_z \E \bd z X.
\end{align*}
As $\dm X = \argmin_z \E \bd z X$, it follows that $\E \bd{\dm X} X = \min_z \E \bd z X$, concluding the proof.
\end{proof}

\propbias*
\begin{proof}
  For any one-hot label $y \in \{0, 1\}$ and probability vector $x$, we have $\kl(y \| x) = \log x_y$, and $\kl[1-y \|x] = \log(1-x_y)$. As $x \to \log 1-x$ is decreasing, it suffices to prove that there exists a distribution $\mathcal D$ such that $\kl[y \| \dm \hat X] \neq \kl[y \| \dm X]$. In fact, it suffices to prove the existence of a distribution $\mathcal D$ such that $\dm X \neq \dm \hat X$.
  
  For the cross-entropy loss, we know\footnote{See, \eg, ~\citep{yang2020rethinking}.} that $\dm X = \textup{softmax}(e^{\E \log X})$. Let $\mathcal D$ be the distribution that assigns equal probability to $x=(0.8, 0.2)$ and $x=(0.6, 0.4)$, and is zero elsewhere. The equivalent ensemble distribution assigns $1/4$ probability to $(0.8, 0.2)$ and $(0.6, 0.4)$, and $1/2$ probability to $(0.7, 0.3)$. A simple numerical computation then shows that $\dm X \neq \dm \hat X$, concluding our proof. 
\end{proof}
\propdualbias*
%\begin{repproposition}{prop:dual-bias}
%Let $D_F$ be a Bregman divergence over $S \times S$, and for any $x_1, %\ldots, x_n \in S$ define 
%  \begin{equation}    
%   \label{eq:ensemble2}
%    \hat x = \Big(\frac 1n \nlsum_i x_i^*\Big)^* = \argmin_z\frac 1n \nlsum_i D(z \| x_i).
%  \end{equation}
%  Ensembling by drawing $n$ models in \iid fashion and combining them by \eqref{eq:ensemble} preserves the bias $D(y \| \dm X)$ and reduces the variance $\E D(\dm X \| X)$.
%\end{repproposition}
\begin{proof}
To preserve bias, it suffices to have $\dm \hat X = \dm X$. By definition of $\hat X$, we have
\[\dm \hat X = \Big(\E \hat X^* \Big)^* = \Big(\E \Big[\frac 1n \nlsum_i X_i^*\Big]\Big)^* = (\E X^*)^* = \dm X.\]
We now focus on the variance. Using the fact that $D_{F}[p\| q] = D_{F^*}[q^*\| p^*]$ \citep[Chapter 11]{cesa2006prediction}, we have
\begin{align*}
    \E D_F(\dm \hat X \| \hat X) &= \E D_F[\dm X \| \hat X] \\
                               &= \E D_{F^*}[\hat X^* \| (\dm X)^*] \\    
                               &= \E D_{F^*}\Big[\frac 1n \nlsum_i  X_i^* \| (\dm X)^*\Big] \\
                               &\overset{(a)}{\le} \frac 1n \nlsum_i \E D_{F^*}\Big[X_i^* \| (\dm X)^*\Big] \\
                               &\le \frac 1n \nlsum_i \E D_{F}\Big[\dm X \| X_i\Big] \\
                               &\le D_F[\dm X \| X].
\end{align*}
where $(a)$ follows from the convexity of $D_{F^*}$ in its first argument.
\end{proof}

% \twoclass*
% \begin{proof}
% When $n=2$, inequality (\ref{eq:arith-geom}) holds if and only if 
% \[ (1-x_1)(1-x_2)(x_1+x_2)^2 - x_1x_2(2-(x_1 + x_2))^2 = (x_1-x_2)^2(1 - (x_1 + x_2)) \ge 0.\]
% Thus, the primal loss is smaller than the dual loss if and only if $x_2  \le 1-x_1$.
% \end{proof}

%\begin{repcorollary}{cor:var-ensemble}
%Let $f$ be an ensemble model built as $f = \frac 1n \sum_i f^{(i)}$, where ensemble members $f^{(1)}, \ldots, f^{(n)}$ are sampled (potentially in non-\iid fashion) from a distribution over models. Define
%\[f^*_j(x) := \frac 1{Z(x)} \exp\left(\E_{f^{(1)}, \ldots, f^{(n)} \sim \mathcal H} \log \frac 1n \nlsum_i f^i_j(x)\right),\]
%where $Z(x)$ is a normalization constant ensuring $\|f^*(x)\|_1 = 1$. Then, the variance in Theorem~\ref{thm:ce-bv} takes the form 
%\[\V_2\ce[y \| f(x)] = - \log Z(x) = -\log\left(\nlsum_j \exp\left(\expect\left[ \log \tfrac1n \nlsum_i f^{(i)}(x)_j\right]\right)\right).\]
%\end{repcorollary}

\section{Experimental Details}
\label{app:experimental}
The models used in this work are wide residual networks (WRN-28-10)~\citep{wrn} with the cross-entropy loss unless specified otherwise. We train models with SGD + momentum to optimize the cross-entropy loss. We use the learning rate schedule, batch size, and data augmentations specified in the deterministic baseline provided by~\citet{nado2021uncertainty}.

As hyperdeep ensembles require the use of a validation set to form the ensemble, we nonetheless set aside  10\% of the training data for validation purposes across all methods. 

For hyperdeep ensembles, we exactly replicated the experimental setup from \cite{wenzel2020hyperparameter}. The only difference is we used $0.9$ fraction of the CIFAR dataset as training data and $0.1$ as validation. For each dataset, we train multiple ($1000$) neural networks where first multiple hyperparameters are chosen randomly from each neural network and then 10 random seeds are chosen for each setting of the hyperparameter. The hyperparameters that are varied are the same as in \cite{wenzel2020hyperparameter} and include $\ell_2$ regularization for various layers and label smoothing.

As in \cite{wenzel2020hyperparameter}, we used the greedy selection strategy to form the ensemble where each ensemble member is chosen greedily based on which member reduces the validation cross entropy loss the most. We trained a total of $1000$ models and then divided randomly into groups of $10$ where $10$ ensemble models were formed where each model was formed by the greedy selection strategy on one group of $100$ models. These $10$ ensemble models were used to do the bias variance decomposition where each decomposition used $5$ models and the bias variance values were averaged over $2$ runs. 

%For CIFAR-10, hyper deep ensembles do not change the error and accuracy too much compared to vanilla deep ensembles which matches with the observation in \cite{wenzel2020hyperparameter}. But, it is still interesting to notice that hyper deep ensembles lead to a lower bias and higher variance as compared to deep ensembles. 
\section{Partitioned estimates of the bias and variance}
\label{app:partition}
\begin{figure}[!h]
\begin{subfigure}{.45\textwidth}
\centering
 \includegraphics[width=.8\textwidth]{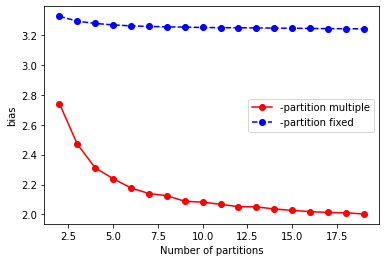}
 \caption{Bias}
\end{subfigure}
\begin{subfigure}{.45\textwidth}
\centering
 \includegraphics[width=.8\textwidth]{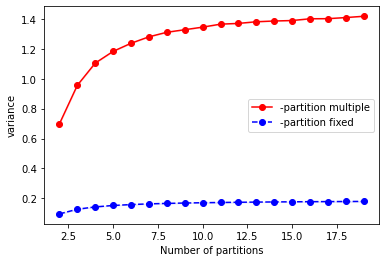}
 \caption{Variance}
\end{subfigure}
\caption{Bias and variance of a smaller WRN-16-5 over the CIFAR-100 dataset. We create 20 partitions of the CIFAR-100 dataset, and estimate the bias either by conditioning on a partition (partition fixed), or by including the partition into the expectations that define the bias and variance (partition multiple). We see that it takes $\ge$ 10 partitions for the estimates of bias and variance to begin converging, and that the converged values appear to still show the bias dominating the variance.}
\label{fig:partitions}
\end{figure}
% \label{app:bv-estimate}
% To estimate the bias and the variance, we need to compute the dual mean predictor $f^*(x) = \dm[f(x)] = \textup{softmax}(\exp(\E \log f(x)))$. We estimate this quantity using $N$ random samples $\{f_i\}_{i=1}^N$ of models from the underlying distribution. This distribution could be over the models with a random data set and a random initialization seed for estimating the true bias and variance values or could be just over a random seed when estimating the conditional values. We use the estimators for the bias and the variance.
% \[ \hat{\text{bias}} = \ce(y \| \hat{f}^*(x))\]
% \[ \hat{\text{variance}} = \frac{1}{N}\sum_{i=1}^N\kl(\hat{f}^*(x) \| f_i(x))\]
% where the estimator of the mean predictor $\hat{f}^*(x)$ is computed as follows.
% \[ \hat{f}^*(x) = \textup{softmax}(\exp(\frac{1}{N}\sum_{i=1}^N \log f_i(x)))\]
% Note that this is not an unbiased estimato. From Figure~\ref{fig:results-bv-bootstrap}, we can see the effect of increasing $N$ on the values of bias and variance and $N = 20$ leads to stable values of the bias and the variance. Hence, for most of our experiments, we choose $N=20$ unless specified otherwise.

\clearpage
\section{Depth and width experiments on CIFAR-10}
\label{app:experiments}
%As we see, that ensembling across multiple depths leads to a noticeable decrease in the error for CIFAR-10. As we from the second plot from the left, the primarily drop in error is coming from the decrease in the bias and variance is slightly higher compared to the fixed depth case. 

\begin{figure}[h]
    \centering
    \includegraphics[width=.9\textwidth]{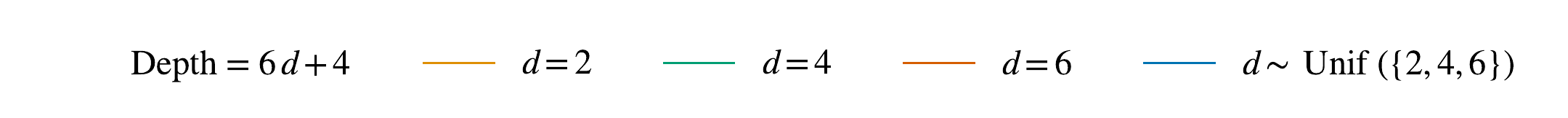}
     \begin{subfigure}{.32\textwidth}
       \centering
       \includegraphics[width=\textwidth]{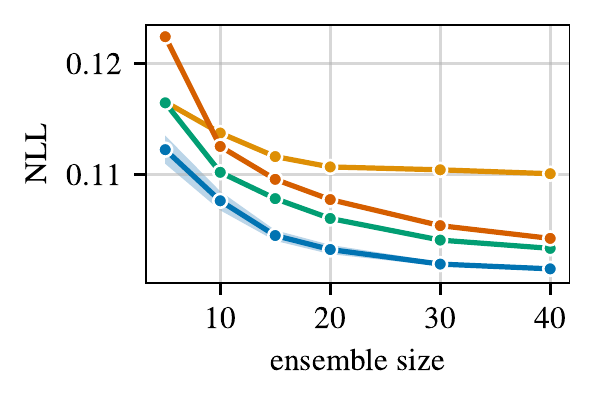}
       \caption{Cross entropy loss}
     \end{subfigure}
     \begin{subfigure}{.32\textwidth}
       \centering
       \includegraphics[width=\textwidth]{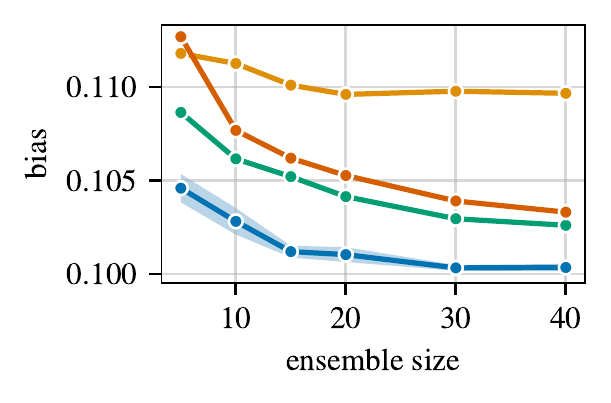}
       \caption{Bias}
     \end{subfigure}
     \begin{subfigure}{.32\textwidth}
         \centering
         \includegraphics[width=\textwidth]{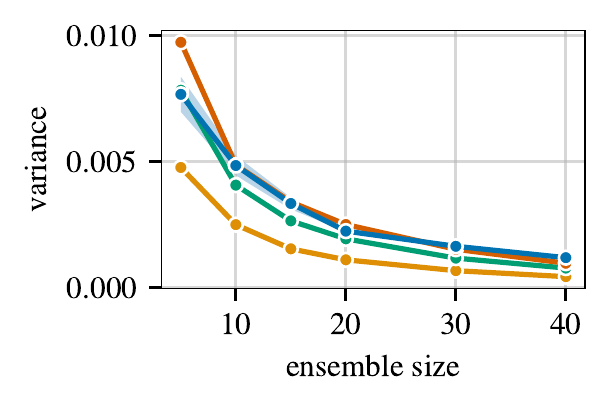}
         \caption{Variance}
     \end{subfigure}
    \caption{\small Plots for bias, variance, error and accuracy with increasing ensemble size for CIFAR-10 when multiple networks are trained with different initial random seeds and different depths and then ensembled together in probability space. The line with d=d denotes the setting where different models with depth as $6d+4$ are ensembled together. The line with $d=\text{Unif}[2,4,6]$ denotes the setting where we first randomly sample a depth from $[16,28,40]$ and then randomly sample a model with that depth. }
    \label{fig:depths-cifar10}
\end{figure}

\begin{figure}[h]
\centering
\includegraphics{figures/widths/width_ensembles_legend.pdf}

     \begin{subfigure}{.32\textwidth}
         \centering
         \includegraphics[width=\textwidth]{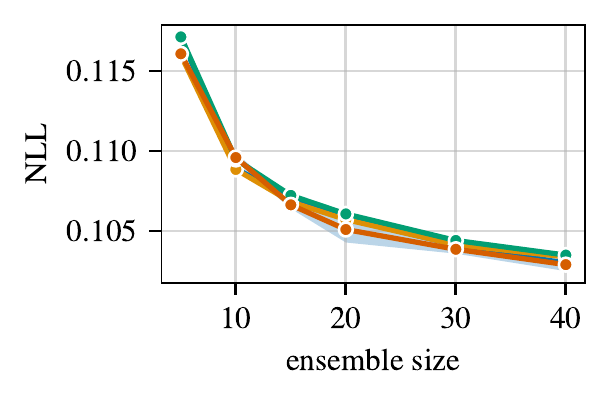}
         \caption{Cross entropy loss}
     \end{subfigure}
     \begin{subfigure}{.32\textwidth}
         \centering
         \includegraphics[width=\textwidth]{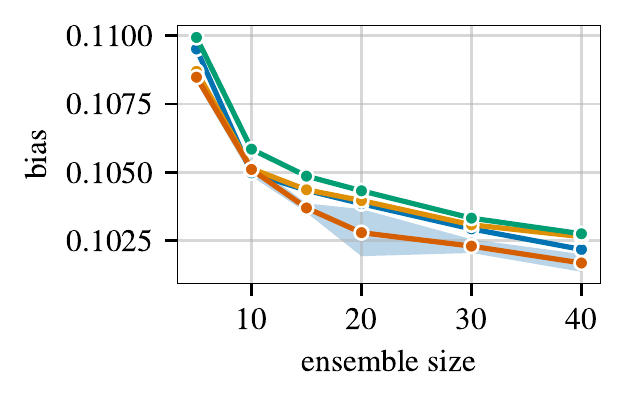}
         \caption{Bias}
     \end{subfigure}
     \begin{subfigure}{.32\textwidth}
         \centering
         \includegraphics[width=\textwidth]{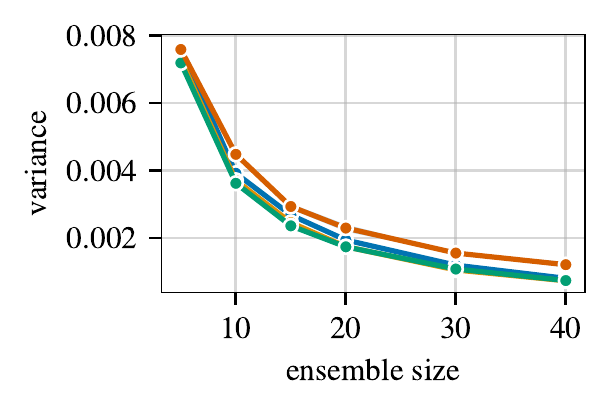}
         \caption{Variance}
     \end{subfigure}
\caption{Plots for bias, variance, error and accuracy with increasing ensemble size for CIFAR-10 when multiple networks are trained with different initial random seeds and different widths and then ensembled together in probability space. The line with width=w denotes the setting where different models with width as w are ensembled together. The line with width=Unif$[10,20,30]$ denotes the setting where we first randomly sample a width from [10, 20, 30] and then randomly sample a model with that width. }
\label{fig:widths-cifar10}
\end{figure}

% \begin{figure}
%      \begin{subfigure}{.24\textwidth}
%       \centering
%       \includegraphics[scale=0.25]{figures/hyper/cifar10/clean/error.png}
%       \caption{Cross entropy loss}
%      \end{subfigure}
%      \begin{subfigure}{.24\textwidth}
%          \centering
%          \includegraphics[scale=0.25]{figures/hyper/cifar10/clean/bias.png}
%          \caption{Bias}
%      \end{subfigure}
%      \begin{subfigure}{.24\textwidth}
%          \centering
%          \includegraphics[scale=0.25]{figures/hyper/cifar10/clean/variance.png}
%          \caption{Variance}
%      \end{subfigure}
%      \begin{subfigure}{.24\textwidth}
%          \centering
%          \includegraphics[scale=0.25]{figures/hyper/cifar10/clean/accuracy.png}
%          \caption{Accuracy}
%      \end{subfigure}
%     \caption{\small Plots for bias, variance, error and accuracy with increasing ensemble size for CIFAR-10 when multiple networks are trained with different initial random seeds and different hyperparameters and then ensembled together in probability space. The different models are selected greedily in hyper deep ensembles. }
%     \label{fig:hyperdeep-cifar10}
% \end{figure}

\end{document}